\documentclass[10pt]{iopart}
\usepackage{url}
\usepackage{amsthm}
\usepackage{iopams} 
\usepackage{cite}
\newtheorem{definition}{Definition}
\newtheorem{theorem}{Theorem}
\newtheorem{proposition}{Proposition}
\usepackage{amssymb}
\usepackage{mathrsfs}
\usepackage{tikz}
\begin{document}

\title[Periodic Spectral Ergodicity]{Periodic Spectral Ergodicity: A Complexity Measure for Deep Neural Networks and Neural Architecture Search}

\author{Mehmet S\"uzen}
\address{Member ACM and IoP}
\ead{suzen@acm.org;mehmet.suzen@physics.org}
\author{J. J. Cerd{\`a}}
\address{Department of Physics, University of the Balearic Islands, Crta. de Valldemossa, km 7.5, 07122 Palma (Illes Balears), Spain}
\ead{jj.cerda@uib.cat}
\author{Cornelius Weber}
\address{Department of Informatics, Faculty of Mathematics and Natural Sciences, University of Hamburg, Germany}
\ead{weber@informatik.uni-hamburg.de}
\vspace{10pt}
\begin{indented}
\item[]January 2020
\end{indented}

\begin{abstract}
Establishing associations between the structure and the generalisation ability of deep neural
networks (DNNs) is a challenging task in modern machine learning. Producing solutions to
this challenge will bring progress both in the theoretical understanding of DNNs and in
building new architectures efficiently. In this work, we address this challenge by developing
a new complexity measure based on the concept of {Periodic Spectral Ergodicity} (PSE) originating 
from quantum statistical mechanics. Based on this measure a technique is devised to quantify
the complexity of deep neural networks from the learned weights and traversing the network connectivity 
in a sequential manner, hence the term cascading PSE (cPSE), as an empirical complexity measure. 
This measure will capture both topological and internal neural processing complexity simultaneously.
Because of this cascading approach, i.e., a symmetric divergence of PSE on the consecutive layers, 
it is possible to use this measure for Neural Architecture Search (NAS). We demonstrate 
the usefulness of this measure in practice on two sets of vision models, ResNet and VGG, and sketch 
the computation of cPSE for more complex network structures.
\end{abstract}

%
\vspace{2pc}

%
%
%
\newpage

\section{Introduction}
Complexity measures appear in multiple fields from physics to medicine in the design of
artificial systems, the understanding of natural phenomena and the detection of signals \cite{pincus91a, 
shiner1999a, bandt2002a, zurek18a}. One of the prominent example systems requiring robust complexity 
measures appear to be machine learning systems. Their recent success in producing learning systems 
exceeding human ability in some tasks is attributed to deep learning \cite{schmidhuber15a, lecun2015a, mnih15a}, 
i.e. Deep Neural Networks (DNNs). However, understanding complexity of DNNs both in structure and 
from learning theory perspective lagged behind their practical engineering success. LeCun \cite{lecun2017my} 
has made an analogy that this situation resembles lack of theory of thermodynamics and the success of 
thermal machines in early industrial revolution.

Sophisticated neural network architectures used in deep learning are still built by human experts
who are usually highly mathematically minded. Understanding the guiding principles in such
design process will help experts to increase their efficiency in producing
superior architectures. Neural Architecture Search (NAS) methods have already been
proposed \cite{elsken19a} in this direction. NAS methods usually operate on the search space,
with a search strategy and a guiding performance estimation strategy.  Embedding complexity
measures in both search strategy and performance estimation strategy could accelerate the NAS frameworks.
Complexity measures for supervised classification were suggested \cite{ho2002a}, however these
measures do not specifically address deep learning architectures. Recently a complexity measure for
characterising deep learning structural complexity is proposed based on topological
data analysis \cite{bianchini2014a, rieck2019a}, which can be used in NAS. It requires
embedding the computation of the measure into the learning algorithm.

In this work, we extend a definition of spectral ergodicity for an ensemble of
matrices \cite{suzen17a} to handle different sized matrices within the same ensemble. This is
possible by defining a periodicity on the eigenvalue vectors via a method, so
called {\it Periodic Spectral Ergodicity (PSE)} on the learned weights. Our main
mathematical object is defined as a complexity measure for the neural network based on PSE.
This measure reflects structure of the network with a cascading computation of PSEs
on the consecutive layers. This leads to a complexity of DNNs measure called cascading PSE (cPSE).
A reason why deep neural networks generalise better with over-parametrisation is a challenge
for generic setting \cite{behnam19a}, cPSE addresses this as well.

The usage of spectral properties of neural networks, i.e., learned weight matrices
or Hessian matrix, has appeared recently in novel works in an attempt to build a strong theoretical foundation
for DNNs \cite{pennington17a, sagun18a, pennington18a, martin18a,martin19a, martin19b} and extracting feature
interaction \cite{tsang17a}. Current work follows a similar ethos.

\section{Periodic spectral ergodicity}

Spectral ergodicity originates from quantum statistical mechanics \cite{jackson01a}.
In a recent inception it is attributed as a reason why deep neural networks
perform in high accuracy, demonstrated on random matrix ensembles as surrogate to weight matrix
ensembles \cite{suzen17a}. In that work, only fixed size complex matrix ensembles were
used to measure the ensemble's approach to spectral ergodicity. Changing matrix dimensions
in ensembles corresponds to increasing the size of a single layer, as an interpretation.
Circular ensembles were used in that study as surrogate matrices having close
to unit spectral radius. In trained neural network's weight matrices, it is known that unit spectral 
radius prevents gradient instabilities. However informative, this approach was short of
direct usage in real practical networks even though it has given empirical evidence of usefulness
of spectral ergodicity.  One shortcoming was how to handle different size weight matrices and the second one was
how to transform complex layer structures like multi-dimensional convolutional units to a weight matrix.

Having multi-dimensional connections between two consecutive layers leads to a tensor representation
of the deep neural network. For this reason, we defined a {\it Layer Matrix Ensemble}  $\mathscr{L}^{m}$,
that is coming from $m$ layer connections via mapping from multi-dimensional connections, i.e.,
trained weights, to two-dimensional square matrices $X_{l}$ of size $N_{l} \times N_{l}$
where $N_{l} \ge 2$, see Definition \ref{LME}. In the case of linear connections, this mapping
acts as a unit transformation. This would enable us to generate a layer matrix ensemble or simply
weight matrices of any trained deep neural network architecture.

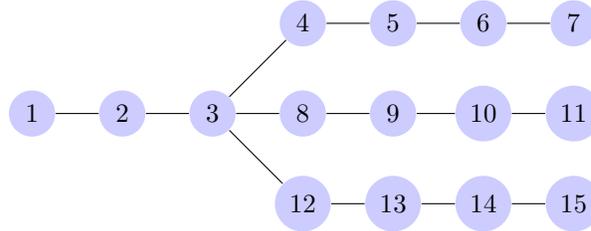
\begin{figure}
  \centering
\begin{tikzpicture}
  [scale=.4,auto=left,every node/.style={circle,fill=blue!20}]
  \node (n1) at (1,6)  {1};
  \node (n2) at (4,6)  {2};
  \node (n3) at (7,6)  {3};
  \node (n4) at (10,9){4};
  \node (n5) at (13,9){5};
  \node (n6) at (16,9){6};
  \node (n7) at (19,9){7};
  \node (n8) at (10,6) {8};
  \node (n9) at (13,6) {9};
  \node (n10) at (16,6){10};
  \node (n11) at (19,6){11};
  \node (n12) at (10,3){12};
  \node (n13) at (13,3){13};
  \node (n14) at (16,3){14};
  \node (n15) at (19,3){15};
  \foreach \from/\to in {n1/n2, n2/n3, n3/n4, n3/n8, n3/n12, n8/n9, n9/n10, n10/n11,
                        n4/n5, n5/n6, n6/n7, n12/n13, n13/n14, n14/n15}
    \draw (\from) -- (\to);
\end{tikzpicture}
  \caption{
           A sample arbitrary architecture {\it a fork}. In the fork each edge represents an arbitrary DNN layer.
           This architecture is a special case of fully connected feedforward architecture,
           certain connections are removed.
          }
\label{fork}
\end{figure}

\begin{definition}{Layer Matrix Ensemble $\mathscr{L}^{m}$} \\
\label{LME}
The weights $W_{l} \in \mathbb{R}^{p_{1} \times p_{2} \times ... \times p_{n}}$ are obtained from
a trained deep neural network architecture's layer $l$ as an $n$-dimensional Tensor. A Layer
Matrix Ensemble $\mathscr{L}^{m}$ is formed by transforming $m$ set of weights $W_{l}$ to square matrices
$X_{l} \in \mathbb{R}^{N_{l} \times N_{l}}$, that $X_{l} = A_{l} \cdot A_{l}^{T}$ and
$A_{l} \in \mathbb{R}^{N_{l} \times M_{l}}$ is marely a stacked up version of
$W_{l}$ where $n > 1$, $N_{l}=p_{1}$, $M_{l}= \prod_{j=2}^{n} p_{j}$ and
$p_{j},n, m, N_{l}, M_{l}, j \in \mathbb{Z}_{+}$. Consequently $\mathscr{L}^{m}$ will
have $m$ potentially different $N_{l}$ size square matrices $X_{l}$ of at least
size $2 \times 2$.
\end{definition}

The original definition of spectral ergodicity for DNNs \cite{suzen17a} relies on
the same size squared matrices in the matrix ensemble. We proceed with applying
 periodic or cyclic conditions to all eigenvalue vectors, obtaining periodic
eigenvalue vectors  $\mathscr{E}^{m}$, see Definition \ref{pev}. This preprocessing
step produces a suitable dataset extracted from trained neural network weights, preserving
layer order, that can be used in computing approach to spectral ergodicity.

\begin{definition}{Periodic eigenvalue vectors $\mathscr{E}^{m}$} \\
\label{pev}
Given layer matrix ensemble $\mathscr{L}^{m}$, the set of eigenvalue vectors
of length $N_{l}$ is computed, eigenvalues of $X_{l}$ as ${\it e_{l}}$. Naturally, not all
vectors will be same length. In order to align the different size vectors, first we
identify the maximum length eigenvalue vector, $l_{max} = \max | {\it e}_{l}|$. Then, we
align all other eigenvalue vectors by applying cyclic boundary condition, meaning
that replicating them up to length $l_{max}$, simply N hereafter.
\end{definition}

Given layer $L$ and $N$ eigenvalues, a spectral density $\rho_{j}$ is computed for all $j$ layers from first to the
layer $L$ using periodic eigenvectors $\mathscr{E}^{m}$. Here a layer corresponds to the weight matrices
in the layer matrix ensemble. Spectral ergodicity at layer $L$ is defined as a distribution extracted from spectral densities,
\begin{equation}
\label{omega}
\Omega^{L} = \Omega^{L} (b_{k})= \frac{1}{L \cdot N} \sum_{j=1}^{L} \left[ \rho_{j}(b_{k}) - \overline{\rho^{L}(b_{k})} \right]^{2},
\end{equation}
where $b_{k}$ are histogram bin centres, and $j$ represents layer $l \le L$. The mean spectral density up to layer L is given by
$$ \overline{\rho^{L}(b_{k})} = \frac{1}{L} \sum_{j=1}^{L} \rho_{j}(b_{k}).$$

\begin{figure}
  \label{fig:vgg}
  \includegraphics[width=0.5\textwidth]{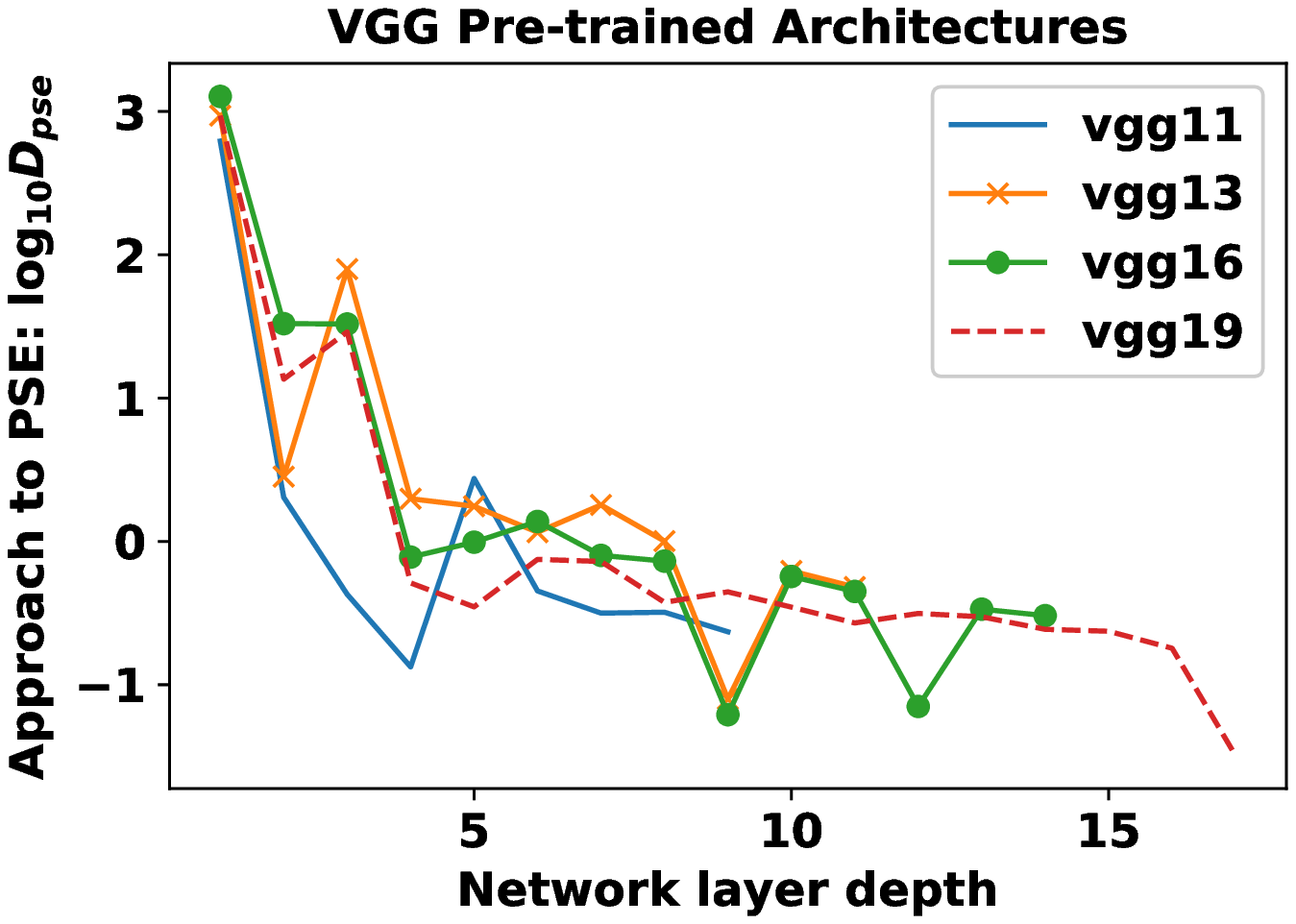}
  \includegraphics[width=0.5\textwidth]{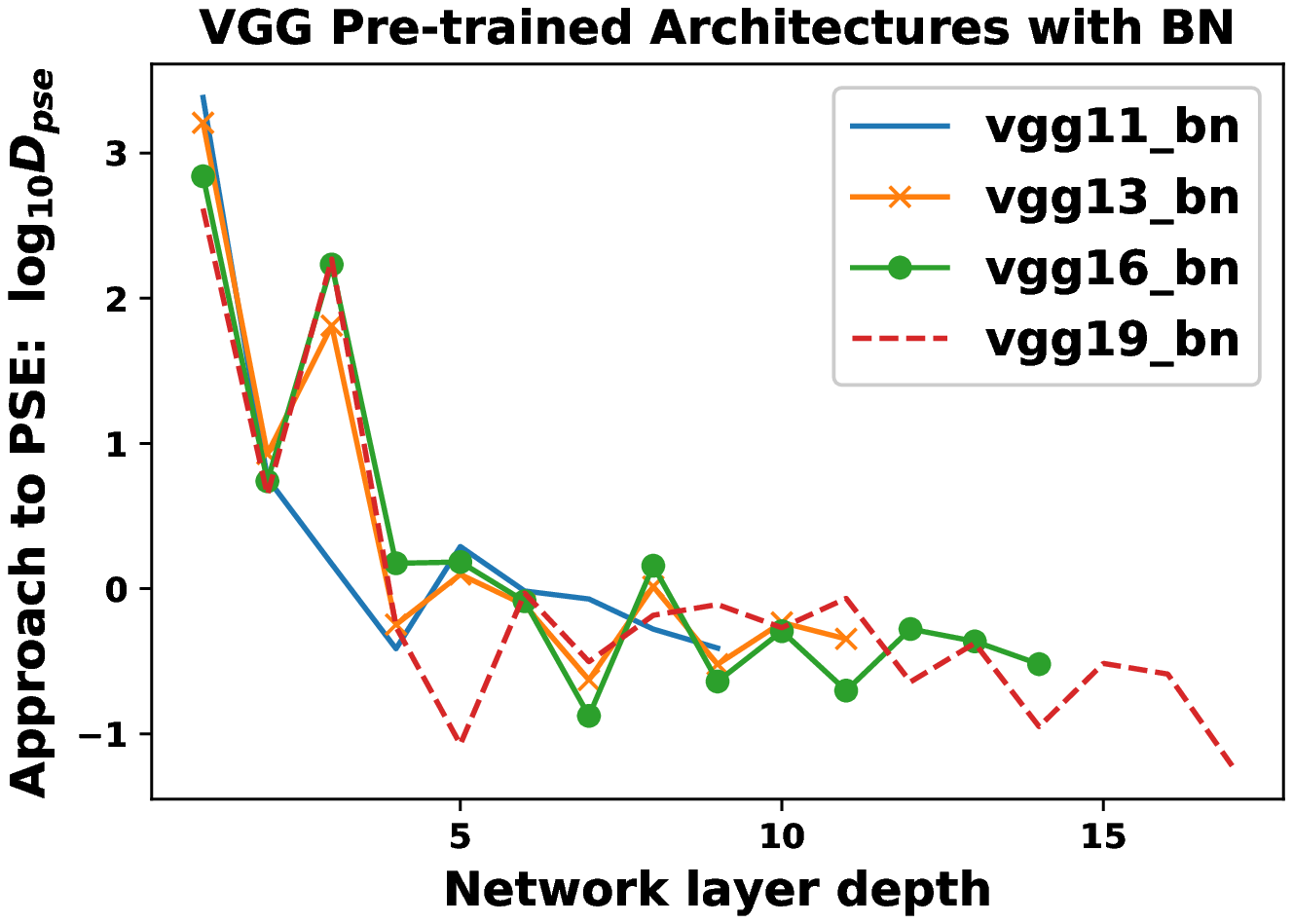}
  \caption{ $D_{pse}$, and for VGG architecture variants (left) and with batch normalisation (right).}
\end{figure}

Since $\Omega^{L}$ alone is a distribution, in order to get an approach to spectral ergodicity over increasing
depth we define a symmetric distance metric between two consecutive layers, $D_{pse}$, as follows,

\begin{equation}
\label{Dse}
D_{pse} = D_{pse}(N_{l}, N_{l+1}) = D_{KL}(\Omega^{l}|\Omega^{l+1}) + D_{KL}(\Omega^{l+1}|\Omega^{l}).
\end{equation}
The terms on the right are Kullbach-Leibler divergence between consecutive layers, one backwards, and they
are defined as follows, sums run over spectral density bins,
$$D_{KL}(\Omega^{l} | \Omega^{l+1}) = \sum_{k} \Omega^{l} \log_{2} \frac{\Omega^{l}}{\Omega^{l+1}},$$
$$ D_{KL}(\Omega^{l+1} | \Omega^{l}) = \sum_{k} \Omega^{l+1} \log_{2} \frac{\Omega^{l+1}}{\Omega^{l}}.$$

It is shown on circular complex matrix ensembles that increasing layer sizes decreases the value of $D_{pse}$.
$D_{pse}$ is defined to be a distance metric for {\it approach to spectral ergodicity} for DNNs.

\subsection{Cascading PSE for feed forward networks}

Complexity of an entire feed forward neural networks, possibily having many complicated connections as in convolution
units, can be identified with the {\it approach to spectral ergodicity} for a given $L$ layered network,
\begin{equation}
\label{cpse_def}
\mathscr{C}^{L} = \frac{1}{L} \sum_{l=1}^{L-1} \log_{10} D_{pse} (N_{l}, N_{l+1}),
\end{equation}
which is identified as cascading PSE (cPSE).

\begin{theorem}{Decreasing cPSE for feedforward networks} \\
\label{decreasing_cpse}
Given $L$ layered DNN and corresponding complexity cPSE $\mathscr{C}^{L}$, adding one
more layer to the DNN almost certainly will not increase the complexity measure cPSE,
$\mathscr{C}^{L} \ge \mathscr{C}^{L+1}$. Hence, lower cPSE implies higher complexity.

\end{theorem}
\begin{proof}
The approach to ergodicity requires decreasing values in distance
metric $\log_{10} D_{pse}$ for well behaved spectra.
Given
$$Dl_{l,l+1} = \log_{10} D_{pse} (N_{l}, N_{l+1})$$
then it follows
$$Dl_{1,2} > Dl_{2,3} > Dl_{3,4} > ...> Dl_{l-1, l} >  Dl_{l,l+1},.. $$ because of the fact that increasing ensemble
size $\Omega_{L}$ should approach to zero monotonically, leading to
decreasing symmetric distance on consecutive layers. This implies $Dl_{1,2} >>Dl_{L-1,L}$. \\ \\

Now, let's say $\mathscr{C}^{L+1}>\mathscr{C}^{L}$ contradicting the theorem,
this yields to
$$ \frac{1}{L} \sum_{l=1}^{L-1} Dl_{l,l+1} <   \frac{1}{L+1} \sum_{l=1}^{L} Dl_{l,l+1} $$
$$ \frac{L}{L^{2}+L} \sum_{l=1}^{L-1} Dl_{l,l+1} <   Dl_{L,L+1} $$
$$ \lambda <   Dl_{L,L+1} $$
is a condradiction as $L$ is very large both $Dl_{L, L+1}$ and $\lambda$ approaches to zero.
So $Dl_{L,L+1}$ can not be greater than $\lambda$ for large $L$.
Hence, the value of cPSE by adding one layer should not increase, $\mathscr{C}^{L+1} \le \mathscr{C}^{L}$. \\
\end{proof}

The assertion that increasing number of layers will decrease the complexity measure
$\mathscr{C}^{L}$ given in Theorem \ref{decreasing_cpse}. Interpretation of this
assertion follows a reverse effect. This means an increasingly complex deep network
will have a smaller $\mathscr{C}^{L}$ value. This agrees with the notion that removing a connection 
decreases complexity \cite{gallant88a}. However, this assertion is closely related to
the learning performance too. Decreasing value of $\mathscr{C}^{L}$ follows performance improvement,
i.e., complexity measure is directly correlated to learning performance. Once a network
reached {\it periodic spectral ergodicity}, very small $D_{pse}$ values, the designed
network's performance will not increase. This is a consequence of Theorem \ref{performance_cpse}.

\begin{proposition}{Performance and cPSE} \\
\label{performance_cpse}
Given an $L$ layered DNN, corresponding complexity cPSE $\mathscr{C}^{L}$ and generalisation
performance measure $\mathscr{P}^{L}$. Given a set of ordered series
$\mathscr{S}^{C} = (L, \mathscr{C}^{L})$ and  $\mathscr{S}^{P} = (L, \mathscr{P}^{L})$.
For a {large enough} set of values of $L$, $\mathscr{S}^{C}$ and $\mathscr{C}^{P}$ are almost perfectly
correlated.
\end{proposition}

It is observed that adding many more layers to a DNNs degrade the performance after certain depth
{\cite{paszke2019a}. We asserted that this discrepancy is a characteristics of
training and numerical difficulty, rather than a generalised theoretical property of
correspondance between learning and the architectures. Performance should not degrade
just to due to depth, if everything else fixed, as a logical inference from
Theorem \ref{performance_cpse}.

\begin{figure}
  \centering
  \label{fig:resnet}
  \includegraphics[width=0.5\textwidth]{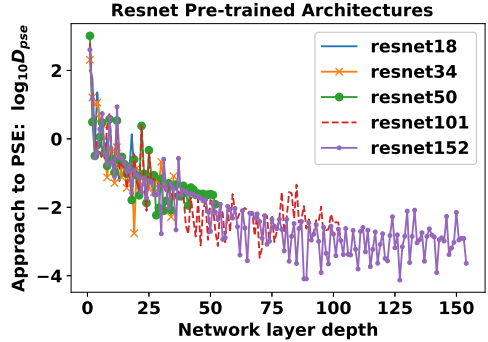}
  \caption{ $D_{pse}$ for ResNet architecture variants.}
\end{figure}

\subsection{Cascading PSE for forward architectures}

The extension of feedforward complexity $\mathscr{C}^{L}$ to any type of connectionist
architecure lies in how to treat branching connections, i.e., a layer connecting to
multiple layers. A basic rule to treat branches, compute complexity up to a branching
point and also up to end of the branch, the difference of the value would give the
branch complexity. By this rule one could compute cPSE for entire arbitrary architectures.

We demonstrate the computation of cPSE for arbitrary architecture with a dummy example.
The fork shaped architecture is identified with 15 layers, braching out at layer 3 to
3 different branches, see Figure \ref{fork}. Using layer ranges as subscripts to indicate
cPSE computation in between those layers, total cPSE for the fork architecture reads,
$$ \mathscr{C}^{L} = \mathscr{C}^{L}_{1-7} + \mathscr{C}^{L}_{1-11} +  \mathscr{C}^{L}_{1-15}
                     - 2 \cdot \mathscr{C}^{L}_{1-3}.$$
Automation of this branching rule might require efficient graph algorithms
for more complex architectures in detecting branches. For recurrent networks,
a transformed version of the architecture to non-recurrent topology is needed to
compute cPSE. Note that the fork architecture is a special case of fully connected feedforward architecture,
certain connections are removed. Fork architecture demonstrates multiple parallel streams.
Hence, layer types, convolutional, or residual connections can be handled by cPSE.

\begin{table}[]
\centering
\begin{tabular}{|l|r|r|r|}
\hline
Architecture  & Top-1 error  &  Top-5 error  & cPSE   \\ \hline
vgg11         & 30.98        &   11.37       & 0.04   \\ \hline
vgg13         & 30.07        &   10.75       & 0.41   \\ \hline
vgg16         & 28.41        &    9.63       & 0.14   \\ \hline
vgg19         & 27.62        &    9.12       &-0.10   \\ \hline
vgg11bn       & 29.62        &   10.19       & 0.38   \\ \hline
vgg13bn       & 28.45        &    9.63       & 0.36   \\ \hline
vgg16bn       & 26.63        &    8.50       & 0.18   \\ \hline
vgg19bn       & 25.76        &    8.15       &-0.07   \\ \hline
resnet18      & 30.24        &   10.92       &-0.19   \\ \hline
resnet34      & 26.70        &    8.58       &-0.74   \\ \hline
resnet50      & 23.85        &    7.13       &-1.03   \\ \hline
resnet101     & 22.63        &    6.44       &-1.77   \\ \hline
resnet152     & 21.69        &    5.94       &-2.29   \\ \hline
\end{tabular}
\
\caption{Classification performance and cPSE of investigated architectures. The correlation between
both classification performances and cPSE for ResNet ($\rho=0.94$), and for VGG ($\rho=0.44$ and $\rho_{bn}=0.93$
with batch normalisation).}
\label{corr}
\end{table}
\section{Experiments}

We tested our assertions and framework on real-life feed forward  architectures designed for
vision tasks: ResNet and VGG variants. Pre-trained weights on ImageNet classification task \cite{paszke2019a} 
are used to build {\it Layer Matrix Ensembles} to compute $D_{pse}$ over increasing layer depth
within the architecture variant. Overall network complexity $\mathscr{C}^{L}$ is computed
via mean $\log_{N} D_{pse}$ values and shows to be decreasing with depth.

\subsection{VGG architectures}

The convolutional network depth investigated for visual image classification \cite{simonyan14a}, so called
VGG architecures.  We have used pre-trained VGG in computing $D_{pse}$, approach to PSE over layers and single $cPSE$ measure
per variant. Results are summarized in Figure \ref{fig:vgg}. We observe the decreasing $D_{pse}$ values increasing depth.
The VGG variants and their corresponding batched normalised versions: {\it vgg11},  {\it vgg13}, {\it vgg16} and {\it vgg19}.
We observed that {\it vgg11}'s cPSE did not obey the monotonicity of cPSE. However batched normalised {\it vgg11} did obey.
VGG11 does not obey monotinicity because it might be too shallow for cPSE to capture the variation of spectral density.

\subsection{ResNet architectures}

A residual neural network (ResNet) brings state-of-the-art results in visual object classification
tasks by introducing layer skip procedures in training \cite{he15a}. We have used pre-trained
variants of ResNet:  {\it resnet18}, {\it resnet34}, {\it resnet50}, {\it resnet101} and
{\it resnet152} in computing $D_{pse}$, approach to PSE over layers and single $cPSE$ measure
per variant.

\section{Emprical evidence for correlation to performance}

We have compared cPSE values for each architecure from our ResNet and VGG variants
against their misclassification test errors. We found very strong correlation between
classification performance on a test dataset and cPSE for the ResNet and VGG.  Correlations remain the
same for both top-1 and top-5 errors. We attribute the drop in correlation for
VGG poorer training approach and the effect of batch normalisation.
The summary is given in Table \ref{corr}, classification performance and cPSE of investigated architectures.
The correlation between both classification performances and cPSE for ResNet $\rho=0.94$
for VGG $\rho=0.44$ and $\rho_{bn}=0.93$ with batch normalisation. Results show 
empirical evidence supporting Proposition \ref{decreasing_cpse}. The higher the depth the smaller the cPSE
complexity measure. 

\section{Conclusion}

We developed a complexity measure for arbitrary neural network architecures
based on {\it spectral ergodicity} defined on learned weights among layers so called
cPSE measure. This measure has given consistent numerical results on the ResNet and VGG
architectures. We also sketch how to compute cPSE with branched networks, i.e. arbitrary
architectures. cPSE provided a quantitative relation between architecture and learning
performance in our test. Moreover, we provided mathematical properties of cPSE, such
as monotonicity.

In this work, we explain the success of deeper neural networks via theoretically
founded structural complexity measure. Going beyond theoretical understanding,
proposed complexity measure for DNNs is simple to implement and it can be directly
used in practice with a minimal effort. Usage of our complexity measure can also be
embedded into NAS methods.  For this reason, we sketch how to compute cPSE for
arbirarily branched networks. Search strategies can use cPSE in deciding how to grow
a network or in genetic algorithms settings how to mutate to a new structure 
leveraging the cPSE.

We also assert that cPSE can provide hints to explain why deep neural networks generalize 
better with over-parametrization. 

\section*{Acknowledgements and Notes}
We would like to express our gratitute to Charles Martin for pointing us out the
usage of pre-trained weights from $pytorchvision$ and thank the PyTorch core team \cite{paszke2019a}
for bundling these datasets neatly. We did not practice so called HARKing 
in this work: HARK issues in machine learning research \cite{gencoglu19a}. 
Propositions are asserted first before generating the data. The supplementary software code
$cpse\_manuscript\_results\_initial\_dev.ipynb$ is provided for reproducing the results. The cPSE measure
is also available in Bristol python package \cite{bristolpy} and an example usage is 
also provided in $cpse\_bristol\_package\_usage.ipynb$ found in the code supplement.
\\ \\
Authors do not have any competing or other kind of conflict of interests.
\\ \\
\bibliography{pse}

\begin{thebibliography}{10}

\bibitem{pincus91a}
Steven~M Pincus.
\newblock Approximate entropy as a measure of system complexity.
\newblock {\em Proceedings of the National Academy of Sciences},
  88(6):2297--2301, 1991.

\bibitem{shiner1999a}
John~S Shiner, Matt Davison, and Peter~T Landsberg.
\newblock Simple measure for complexity.
\newblock {\em Physical Review E}, 59(2):1459, 1999.

\bibitem{bandt2002a}
Christoph Bandt and Bernd Pompe.
\newblock Permutation entropy: a natural complexity measure for time series.
\newblock {\em Physical Review Letters}, 88(17):174102, 2002.

\bibitem{zurek18a}
Wojciech~H Zurek.
\newblock {\em Complexity, entropy and the physics of information}.
\newblock CRC Press, 2018.

\bibitem{schmidhuber15a}
J{\"u}rgen Schmidhuber.
\newblock Deep learning in neural networks: An overview.
\newblock {\em Neural networks}, 61:85--117, 2015.

\bibitem{lecun2015a}
Yann LeCun, Yoshua Bengio, and Geoffrey Hinton.
\newblock Deep learning.
\newblock {\em Nature}, 521(7553):436--444, 2015.

\bibitem{mnih15a}
Volodymyr Mnih, Koray Kavukcuoglu, David Silver, Andrei~A Rusu, Joel Veness,
  Marc~G Bellemare, Alex Graves, Martin Riedmiller, Andreas~K Fidjeland, Georg
  Ostrovski, et~al.
\newblock Human-level control through deep reinforcement learning.
\newblock {\em Nature}, 518(7540):529, 2015.

\bibitem{lecun2017my}
Yann LeCun.
\newblock {My take on Ali Rahimi's “Test of Time” award talk at NIPS}.
\newblock {\em unpublished}.

\bibitem{elsken19a}
Thomas Elsken, Jan~Hendrik Metzen, and Frank Hutter.
\newblock Neural architecture search: A survey.
\newblock {\em Journal of Machine Learning Research}, 20:1--21, 2019.

\bibitem{ho2002a}
Tin~Kam Ho and Mitra Basu.
\newblock Complexity measures of supervised classification problems.
\newblock {\em IEEE Transactions on Pattern Analysis \& Machine Intelligence},
  (3):289--300, 2002.

\bibitem{bianchini2014a}
M.~{Bianchini} and F.~{Scarselli}.
\newblock On the complexity of neural network classifiers: A comparison between
  shallow and deep architectures.
\newblock {\em IEEE Transactions on Neural Networks and Learning Systems},
  25(8):1553--1565, 2014.

\bibitem{rieck2019a}
Bastian Rieck, Matteo Togninalli, Christian Bock, Michael Moor, Max Horn,
  Thomas Gumbsch, and Karsten Borgwardt.
\newblock Neural persistence: A complexity measure for deep neural networks
  using algebraic topology.
\newblock In {\em International Conference on Learning
  Representations\~{}(ICLR)}, 2019.

\bibitem{suzen17a}
Mehmet {S{\"u}zen}, Cornelius {Weber}, and Joan~J. {Cerd{\`a}}.
\newblock {Spectral Ergodicity in Deep Learning Architectures via Surrogate
  Random Matrices}.
\newblock {\em arXiv preprint arXiv:1704.08303}, Apr 2017.

\bibitem{behnam19a}
Behnam Neyshabur, Zhiyuan Li, Srinadh Bhojanapalli, Yann LeCun, and Nathan
  Srebro.
\newblock The role of over-parametrization in generalization of neural
  networks.
\newblock In {\em 7th International Conference on Learning Representations,
  {ICLR} 2019, New Orleans, LA, USA, May 6-9, 2019}, 2019.

\bibitem{pennington17a}
Jeffrey Pennington, Samuel Schoenholz, and Surya Ganguli.
\newblock Resurrecting the sigmoid in deep learning through dynamical isometry:
  theory and practice.
\newblock In {\em Advances in Neural Information Processing Systems}, pages
  4785--4795, 2017.

\bibitem{sagun18a}
Levent Sagun, Utku Evci, V~Ugur Guney, Yann Dauphin, and Leon Bottou.
\newblock {Empirical analysis of the Hessian of over-parametrized neural
  networks}.
\newblock {\em arXiv preprint arXiv:1706.04454}, 2017.

\bibitem{pennington18a}
Jeffrey Pennington, Samuel~S Schoenholz, and Surya Ganguli.
\newblock The emergence of spectral universality in deep networks.
\newblock {\em arXiv preprint arXiv:1802.09979}, 2018.

\bibitem{martin18a}
Charles~H Martin and Michael~W Mahoney.
\newblock Implicit self-regularization in deep neural networks: Evidence from
  random matrix theory and implications for learning.
\newblock {\em arXiv preprint arXiv:1810.01075}, 2018.

\bibitem{martin19a}
Charles~H Martin and Michael~W Mahoney.
\newblock Traditional and heavy-tailed self regularization in neural network
  models.
\newblock {\em arXiv preprint arXiv:1901.08276}, 2019.

\bibitem{martin19b}
Charles~H Martin and Michael~W Mahoney.
\newblock Heavy-tailed universality predicts trends in test accuracies for very
  large pre-trained deep neural networks.
\newblock {\em arXiv preprint arXiv:1901.08278}, 2019.

\bibitem{tsang17a}
Michael Tsang, Dehua Cheng, and Yan Liu.
\newblock Detecting statistical interactions from neural network weights.
\newblock {\em arXiv preprint arXiv:1705.04977}, 2017.

\bibitem{jackson01a}
AD~Jackson, C~Mejia-Monasterio, T~Rupp, M~Saltzer, and T~Wilke.
\newblock Spectral ergodicity and normal modes in ensembles of sparse matrices.
\newblock {\em Nuclear Physics A}, 687(3-4):405--434, 2001.

\bibitem{gallant88a}
Stephen~I. Gallant.
\newblock Connectionist expert systems.
\newblock {\em Commun. ACM}, 31(2):152–169, February 1988.

\bibitem{paszke2019a}
Adam Paszke, Sam Gross, Francisco Massa, Adam Lerer, James Bradbury, Gregory
  Chanan, Trevor Killeen, Zeming Lin, Natalia Gimelshein, Luca Antiga, Alban
  Desmaison, Andreas Kopf, Edward Yang, Zachary DeVito, Martin Raison, Alykhan
  Tejani, Sasank Chilamkurthy, Benoit Steiner, Lu~Fang, Junjie Bai, and Soumith
  Chintala.
\newblock Pytorch: An imperative style, high-performance deep learning library.
\newblock pages 8024--8035, 2019.

\bibitem{simonyan14a}
Karen Simonyan and Andrew Zisserman.
\newblock Very deep convolutional networks for large-scale image recognition.
\newblock {\em arXiv preprint arXiv:1409.1556}, 2014.

\bibitem{he15a}
Kaiming He, Xiangyu Zhang, Shaoqing Ren, and Jian Sun.
\newblock Deep residual learning for image recognition.
\newblock {\em 2016 IEEE Conference on Computer Vision and Pattern Recognition
  (CVPR)}, pages 770--778, 2015.

\bibitem{gencoglu19a}
Oguzhan Gencoglu, Mark van Gils, Esin Guldogan, Chamin Morikawa, Mehmet
  S{\"u}zen, Mathias Gruber, Jussi Leinonen, and Heikki Huttunen.
\newblock Hark side of deep learning - from grad student descent to automated
  machine learning.
\newblock {\em arXiv preprint arXiv:1904.0763}, 2019.

\bibitem{bristolpy}
M.~S\"uzen.
\newblock Bristol python package: Random matrix tools.
\newblock \url{https://pypi.python.org/pypi/bristol}.

\end{thebibliography}
\end{document}